\documentclass[runningheads]{llncs}
\usepackage[T1]{fontenc}
\usepackage{graphicx}
\usepackage{booktabs}
\usepackage[misc]{ifsym}

\usepackage{cite}
\usepackage{threeparttable}
\usepackage{amsmath,amssymb,amsfonts}
\usepackage{algorithm}
\usepackage{algorithmic}
\usepackage{textcomp}
\usepackage{xcolor}
\usepackage{graphicx}
\usepackage{bbm}
\usepackage{bm}
\usepackage{subfigure}
\usepackage{soul}
\usepackage{xspace} 
\usepackage{appendix}
\def\BibTeX{{\rm B\kern-.05em{\sc i\kern-.025em b}\kern-.08em
    T\kern-.1667em\lower.7ex\hbox{E}\kern-.125emX}}

\makeatletter

\newcommand{\Rmnum}[1]{\expandafter\@slowromancap\romannumeral #1@}
\makeatother

\def\framework{SpanGNN\xspace}

\begin{document}
\title{\framework: Towards Memory-Efficient Graph Neural Networks via Spanning Subgraph Training}
\titlerunning{\framework: Towards Memory-Efficient Graph Neural Networks}
%
\author{Xizhi Gu\inst{1} \and
Hongzheng Li\inst{1} \and
Shihong Gao\inst{2} \and
Xinyan Zhang\inst{1} \and
Lei Chen\inst{3} \and
Yingxia Shao\inst{1} \thanks{Corresponding author}
\authorrunning{F. Author et al.}
%
\institute{Beijing University of Posts and Telecommunications \and
Hong Kong University of Science and Technology \and
Hong Kong University of Science and Technology (Guangzhou)\\
\email {\{guxizhi, Ethan\_Lee\}@bupt.edu.cn, sgaoar@connect.ust.hk, xianxx492@bupt.edu.cn, leichen@cse.ust.hk, shaoyx@bupt.edu.cn}}}
\maketitle              
\begin{abstract}
Graph Neural Networks (GNNs) have superior capability in learning graph data. Full-graph GNN training generally has high accuracy, however, it suffers from large peak memory usage and encounters the Out-of-Memory problem when handling large graphs. To address this memory problem, a popular solution is mini-batch GNN training. However, mini-batch GNN training increases the training variance and sacrifices the model accuracy. 
In this paper, we propose a new memory-efficient GNN training method using spanning subgraph, called \framework. 
\framework trains GNN models over a sequence of spanning subgraphs, which are constructed from empty structure. To overcome the excessive peak memory consumption problem, \framework selects a set of edges from the original graph to incrementally update the spanning subgraph between every epoch. To ensure the model accuracy, we introduce two types of edge sampling strategies (i.e., variance-reduced and noise-reduced), and help \framework select high-quality edges for the GNN learning. We conduct experiments with \framework on widely used datasets, demonstrating \framework's advantages in the model performance and low peak memory usage.

\end{abstract}
%
%
%

\section{Introduction}\label{Introduction}
Graph Neural Networks (GNNs) achieve the state-of-the-art performance on graph learning tasks, such as node classification~\cite{kipf2016semi, hamilton2017inductive, velivckovic2017graph}, link prediction~\cite{zhang2018link, schlichtkrull2018modeling} and graph classification~\cite{duvenaud2015convolutional, ying2018hierarchical}. They have been widely used in various domains, like social network analysis~\cite{gnn_sna, gnn_social_rec}, recommendation~\cite{gnn_rec, app_rec3, graph_contrastive_rec}, healthcare~\cite{gnn_health1, gnn_health2}, short-term load forecasting~\cite{short_term_load} and bio-informatics~\cite{gnn_bio, gnn_bio1}. 
Most of GNNs~\cite{kipf2016semi, hamilton2017inductive, velivckovic2017graph, xu2018powerful} follow the message passing paradigm~\cite{gnnmp_icml_2017}, which exploits graph topology and node/edge features simultaneously. In this paradigm, the edge related memory consumption predominantly influences the amount of peak GPU memory~\cite{wang2021empirical}. The edge calculation of GNNs involves four operations, which are collection, messaging, aggregation, and feature updating. The four operations require the storage of intermediate results (e.g., the updated embedding of edge feature and the aggregation of feature embedding), which are used for the gradient calculation in the subsequent backward propagation process. According to the existing empirical studies~\cite{wang2021empirical}, the peak memory consumption can be up to 100 times of the size of dataset itself. As a result, the high memory usage of the edge calculation restricts the GNNs scaling to large graphs.

Since the edge calculation is the main factor of high memory usage, an intuitive idea is to reduce the number of edges for training. Sampling is a standard technique to generate graphs with few edges. It has been well studied in the mini-batch training. Many works~\cite{hamilton2017inductive, ying2018hierarchical, 2018FastGCN, 2019Cluster, graphsaint} use various sampling techniques to create mini-batches, which are subgraphs rooted from a limited number of target nodes. Although mini-batch training is scalable and memory-efficient, it brings in non-negligible training variance and heavily compromises model accuracy. Full-graph GNN training is more accurate than mini-batch training~\cite{ROC_mlsys}. However, the existing complex sampling methods cannot be efficiently adopted to the full-graph GNN training. The sampling step is time-consuming and becomes the efficiency bottleneck for GNN training on large graphs~\cite{wang2021empirical}.
Unlike the sampling technique, DropEdge~\cite{Dropedge} randomly drops edges of the original graph during the full-graph training. It not only reduces the size of peak memory, but also is scalable to large graphs. Nonetheless, DropEdge also suffers from a prominent model accuracy loss as the edge drop ratio increases, especially on large graphs. This limitation arises because DropEdge treats all edges equally and ignores the inherent structure of the original graph. Therefore, {\em how to develop a memory-efficient and accurate full-graph GNN learning method remains unsolved}.


In this paper, we propose \framework to achieve memory-efficient full-graph GNN training while guaranteeing the model accuracy. First, \framework trains GNN models across a sequence of spanning subgraphs, which are constructed from empty structure. Each spanning subgraph contains significantly fewer edges than those present in the original graphs, thus effectively reducing the peak memory footprint. Furthermore, in each training epoch, \framework selects a set of edges from the original graph to incrementally update the spanning subgraph that used in the previous epoch. Meanwhile, the updated spanning graph always satisfies the sparsity constraint defined by the edge ratio $\alpha$ (See the definition in Section~\ref{sec:spangnn}).
Second, to guarantee the model accuracy and training efficiency, we propose a fast quality-aware edge selection method for \framework. We analyze the training variance and gradient noise that inherent in the spanning subgraph training framework, and propose variance-reduced sampling and gradient-noise reduced sampling strategies, respectively, to help \framework selects a set of high-quality edges and guarantee the model accuracy. However, it is expensive to directly apply the above two sampling strategies over large graphs, we introduce a two-step sampling method to speed up the edge selection process. Extensive experiments demonstrate that \framework is capable of saving over 40\% of GPU memory usage without compromising training performance.

Our main contributions are summarized as follows: 
\begin{itemize}
    \item We propose \framework that supports memory-efficient and accurate full-graph GNN training on large graphs. The new method reduces the peak memory usage significantly during the training, meanwhile achieving high model accuracy. 
    \item We introduce a fast quality-aware edge selection method to alleviate the negative impacts caused by spanning subgraph training and ensure training efficiency.
    \item We analyze the connection between \framework and curriculum learning~\cite{bengio2009curriculum}. With the help of quality-aware edge selection, \framework selects edges that are highly beneficial to the learning in priority, and then gradually uses edges with low benefits. 
    \item Experimental results on widely used datasets demonstrate that \framework reduces peak memory usage effectively while guaranteeing that the model accuracy is almost equivalent to the one of full graph training.
\end{itemize}

\section{Preliminary} \label{Preliminary}

\subsection{Graph Neural Networks}
The general matrix formulation of GNN models is as follows:
\begin{equation}
    Z^{(l)}=PH^{(l-1)}W^{(l-1)},\label{GNN1}
\end{equation}
\begin{equation}
    H^{(l)}=\sigma(Z^{(l)}), \label{GNN2}
\end{equation}
where $Z^{(l)}$, $H^{(l)}$, and $W^{(l)}$ represent the intermediate embedding matrix, feature embedding matrix and trainable weight matrix at $l$-th layer, respectively. $\sigma$ is a non-linear activation function, like ReLu. $P$ is the propagation matrix that is transformed from the graph adjacency matrix.

During the backward propagation, the gradient of the loss with respect to $W^{(l-1)}$ is as follows:
\begin{equation}
    \nabla_{W^{(l-1)}}L=\frac{\partial L}{\partial W^{(l-1)}}=(H^{(l-1)})^{T}P^{T}\delta^{(l)},\label{GNN3}
\end{equation}
where $\delta^{(l)}$ denotes the gradient of the loss with respect to $Z^{(l)}$. Then, $W^{(l-1)}$ is updated as follows:
\begin{equation}
    W^{(l-1)}=W^{(l-1)}-\eta \nabla_{W^{(l-1)}}L, \label{GNN4}
\end{equation}
where $\eta$ denotes the learning rate of training.



\subsection{Spanning Subgraph GNN Training}\label{sec:spangnn}
Given a graph $G=(V, E)$, a spanning subgraph $G_s=(V_s, E_s)$ generated from $G$ is a subgraph with vertex set $V$, i.e., $V_s=V$ and $E_s \subset E$~\cite{west_introduction_2000}. 
We define edge ratio $\alpha$ between $G_s$ and $G$ is $\frac{|E_s|}{|E|}$. $\alpha$ represents the degree of edge reduction of a spanning subgraph against the corresponding original graph. The smaller $\alpha$ is, the more edges are deleted, and less memory is demanded for training over the spanning subgraph. 

Spanning subgraph GNN training make GNNs only propagation along the subgraph $G_s$. Therefore, the key GNN operations (E.q.~\ref{GNN1}-\ref{GNN3}) are rewritten as: 

\begin{equation}
    \tilde{Z}^{(l)}=\tilde{P}H^{(l-1)}W^{(l-1)},\label{GNNRE1}
\end{equation}
\begin{equation}
    \tilde{H}^{(l)}=\sigma(\tilde{Z}^{(l)}),\label{GNNRE2}
\end{equation}
\begin{equation}
    \nabla_{W^{(l-1)}}\tilde{L}=\frac{\partial \tilde{L}}{\partial W^{(l-1)}}=(\tilde{H}^{(l-1)})^{T}\tilde{P}^{T}\tilde{\delta}^{(l)},\label{GNNRE3}
\end{equation}
where $\tilde{P}$ is the propagation matrix that is transformed from the spanning subgraph. 
Spanning subgraph GNN training results in approximated node embedding matrix $\tilde{H}^{(l)}$ and the approximated embedding gradients $\tilde{\delta}^{(l)}$. The model accuracy is affected by these approximated intermediate results as well. We will discuss the main factors that influence the model accuracy in Section~\ref{pro edge selection}.

\section{\framework: Memory-Efficient Full-graph GNN Learning}


\begin{figure*}[ht]
    \centering
    \includegraphics[width=\textwidth]{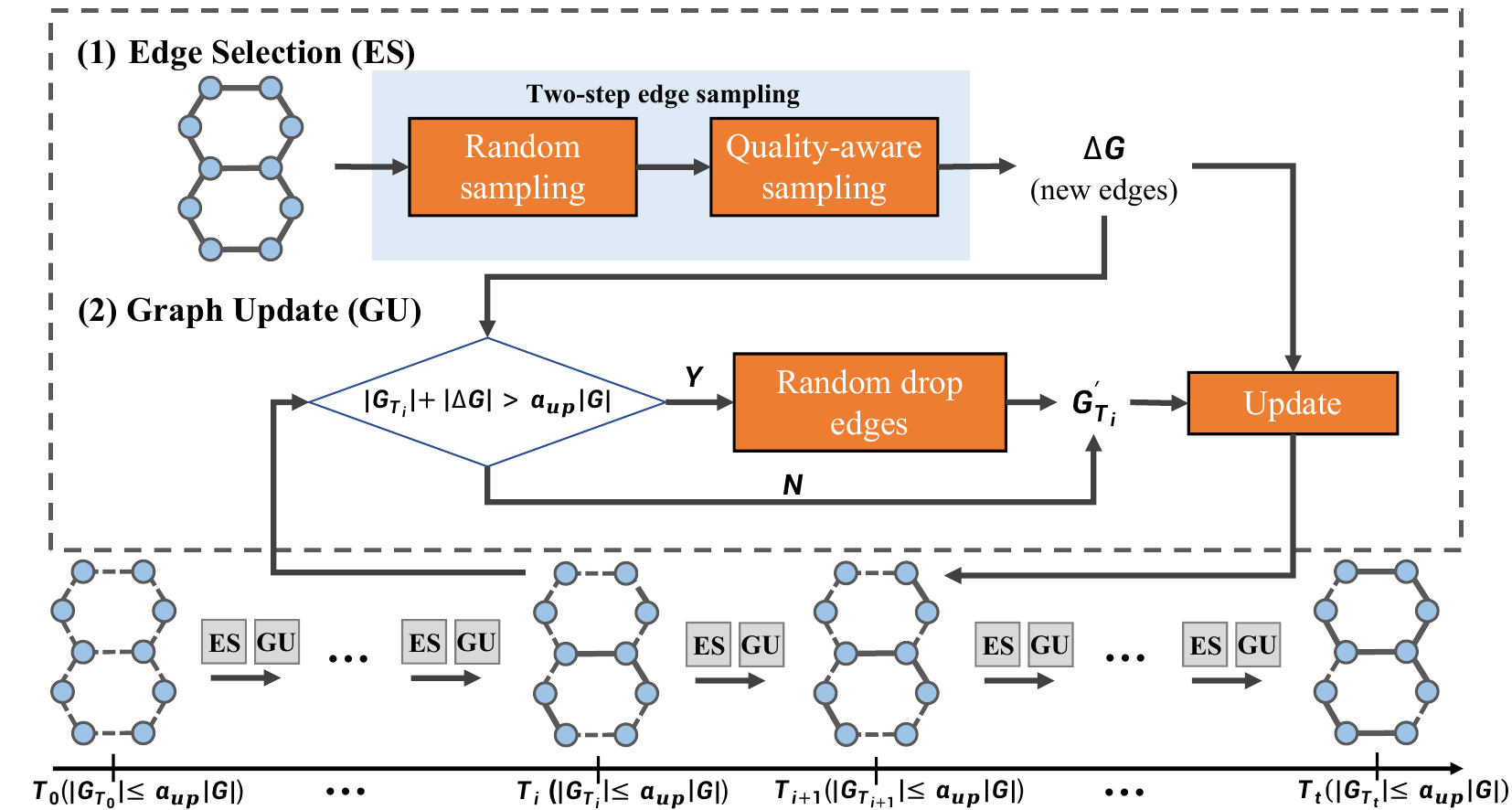}
    \centering
    \caption{The framework of \framework.}
    \label{framework_overview}
\end{figure*}

Figure~\ref{framework_overview} illustrates the overview of \framework. It starts with an empty spanning subgraph $G_{T_0}$ and progressively includes more edges during the training. For every training epoch $T_i$, \framework selects a set of edges from the original graph $G$, updates the spanning subgraph $G_{T_{i-1}}$ with the selected edge set, and generates a new spanning subgraph $G_{T_i}$. Furthermore, to limit the peak memory usage, \framework guarantees that, in each training epoch, the edge ratio $\frac{|E_{T_i}|}{|E|}$ does not exceed $\alpha_{up}$, which is preset by users. According to the definition of edge ratio, the parameter $\alpha_{up}$ implies the upper bound of peak memory usage during the training. Therefore, \framework is able to control the maximal size of peak memory flexibly and is memory-efficient. 

Algorithm~\ref{alg:Framwork} illustrates the whole process of spanning subgraph generation and GNN model training. Starting with an empty spanning subgraph (Line 2), \framework selects $\Delta G$ via quality-aware edge selection in every epoch (Line 4) and updates the spanning subgraph according to the condition $|G_{T_{i}}| + |\Delta G| \geq \alpha_{up}|G|$ (Line 5). $RandomDrop$ is the function that randomly deletes a fixed proportion of edges from the input graph (Line 6). $GraphUpdate$ is the function that merges the new selected edges into the current graph (Lines 7,9).

\begin{algorithm}[ht]
\caption{ Framework of \framework.}
\label{alg:Framwork}
\begin{algorithmic}[1] 
    \REQUIRE The graph $G$, GNN model's parameters $W$, the upper bound of edge ratio $\alpha_{up}$, the ratio of dropping edges $\beta$
    \STATE $i=0$ ;
    \STATE Initialize $G_{T_{0}}$ with empty edge;
    \WHILE{$i \leq t$}
        \STATE $\Delta G = QualityAwareEdgeSelect(G)$;
        \IF{$|G_{T_{i}}| + |\Delta G| \geq \alpha_{up}|G|$}
            \STATE $G_{T_{i}}^{'} = RandomDrop(G_{T_{i}}, \beta)$;
            \STATE $G_{T_{i+1}} = GraphUpdate(G_{T_{i}}^{'}, \Delta G)$;
        \ELSE
            \STATE $G_{T_{i+1}} = GraphUpdate(G_{T_{i}}, \Delta G)$;
        \ENDIF
        \STATE Training GNN model on $G_{T_{i+1}}$ and update model's parameters $W$;
        \STATE $i=i+1$ ;
    \ENDWHILE
\end{algorithmic}
\end{algorithm}


\textbf{Edge Selection.} Edges in the graph contribute differently to the GNN training, so it is important to pick out the most beneficial edges for the training to guarantee the model accuracy. Weighted sampling is a standard approach to select important edges in priority. In this paper, we analyze two types of factors that influence the model accuracy and propose quality-aware edge selection approach in Section~\ref{pro edge selection}. The new edge selection approach adopts variance-reduced sampling strategy and gradient-noise reduced sampling strategy to select high-quality edges. However, directly sampling from the entire graph with non-uniform probability distribution is time-consuming. We further introduce the two-step edge sampling method to speed up the edge selection. In addition, using the quality-aware edge selection approach, the training process of \framework aligns with the principles of curriculum learning (discussed in Section~\ref{cl}), therefore, \framework has high accuracy.

\textbf{Graph Update.} In order to continuously satisfy the edge ratio constraint (i.e., $\frac{|G_{T_i}|}{|G|} \le \alpha_{up}$), we introduce an edge drop step in the graph update. In each training epoch $T_i$, if \framework detects that the new spanning subgraph will violate the edge ratio constraint, it first randomly drops a set of edges in $G_{T_{i-1}}$, then adds the selected edge set to the subgraph $G_{T_{i-1}}$; otherwise, the selected edge set is directly added into the subgraph $G_{T_{i-1}}$. The edge drop step helps \framework ensure the memory-efficiency. More importantly, it improves the diversity of the trained spanning subgraphs and enhances the model accuracy as well.

\section{Fast Quality-aware Edge Selection} \label{pro edge selection}
In this section, we first introduce two types of edge sampling strategies, which are variance-reduced sampling strategy and gradient noise-reduced sampling strategy. Then, we introduce the two-step edge sampling method that optimizes the efficiency of edge selection over large graphs.

\subsection{Variance-minimized Sampling Strategy}

\textbf{The Variance of Aggregated Embedding}.
Since the edges are probability selected in \framework, the spanning subgraph can be treated as a sampled subgraph from the original graph, and the variance of aggregated embeddings in \framework affects the model accuracy and should not be ignored. Similar to the existing works~\cite{graphsaint}, the unbiased estimator of aggregated embeddings without activation and the variance of embeddings estimator can be defined as:
\begin{equation}
    \xi = \sum_{(l)} \sum_{e} \frac{b^{(l)}_{e}}{p_{e}} \mathbbm{1}_{e}^{(l)}
\end{equation}
\begin{equation}
    Var(\xi) = \sum_{e} \frac{(\sum_{l} b_{e}^{(l)})^{2}}{p_{e}} - \sum_{e}(\sum_{l}b_{e}^{(l)})^{2},
\end{equation}
where $p_{e}$ denotes the probability of an edge to be sampled, $b_{e}^{(l)}=P_{v,u}\tilde{x}^{(l-1)}_{u} + P_{u,v}\tilde{x}^{(l-1)}_{v}$, $P$ is the propagation matrix (e.g., the normalized adjacency matrix in GCN), $\tilde{x}$ is feature matrix after linear operation, and $\mathbbm{1}_{e}=1$ if $e$ is in $E_{s}$.

\textbf{Variance Minimization.} To minimize the variance of the aggregated embeddings estimator, we follow the strategy used in GraphSAINT~\cite{graphsaint}. By using the Cauchy-Schwarz inequality, the variance of aggregated embeddings is minimized when $p_{e} \propto |\sum_{l}b_{e}^{(l)}|$, which can be simplified as:
\begin{equation}
    p_{e} \propto P_{v,u} + P_{u,v} = \frac{1}{deg(u)} + \frac{1}{deg(v)}. \label{prob:var}
\end{equation}

The probability $p_{e}$ defined in E.q.~\ref{prob:var} interprets that if two nodes $u, v$ are connected and they have few neighbors, then edge between $u$ and $v$ are more likely to be sampled and to reduce the variance. In other words, such edges will contain more information for the nodes, which is more conducive to the training of the node.

\subsection{Gradient Noise-reduced Sampling Strategy}

\textbf{The Noise of Gradient.} As mentioned before, with the spanning subgraphs, the final learned embeddings change compared to the exact ones. Therefore, the results of loss function and gradient change as well. We define the noise of gradient as the change between $\nabla_{W^{(l-1)}}L$s that is calculated by original graph training and spanning subgraph training. We formulate the noise of gradient as below:
\begin{equation}
    G_{noise} = \nabla_{W^{(l-1)}}\tilde{L} - \nabla_{W^{(l-1)}}L.
    \label{G_noise}
\end{equation}

\textbf{Gradient Noise Reduction.} The noise of gradient slows down the convergence and affects the model accuracy. To solve the problem, we derive a probability distribution for edge sampling that can reduce the upper bound of gradient noise. The probability of an edge $e(u,v)$ is formulated as:
\begin{equation}
    p_{v,u}=\frac{\left\| P_{*,u} \right\|_{2}}{\sum_{(v,u)\in E}\left\| P_{*,u} \right\|_{2}},
    \label{final_prob}
\end{equation}
where $P_{*,u}$ denotes the vector of $node_{u}$'s propagation matrix.. The larger the sampling probability of the edge, the smaller the gradient noise in the training process.

Next, we theoretically analyze the above probability and dig into the upper bound of the expected gradient noise, which is summarized in Theorem~\ref{th1}.

\begin{theorem}\label{th1}
\textbf{Upper bound of the expected gradient noise.} Given the square of Frobenius norm $\left\| P\right\|^{2}_{F}$, $\left\| H^{(l)}\right\|^{2}_{F}$, $\left\| \delta^{(l)}\right\|^{2}_{F}$ are bounded by some constants $B$, $C$, $D$ and the L2-norm $\left\| H^{(l)}W^{(l)} \right\|$ is bounded by constant $\xi$. Assume that the activation function $\sigma$ is $\rho_{\sigma}-Lipschitz$ and the gradient $\nabla_{Z^{(l)}}L$ is $\rho_{Z}-Lipschitz$, then we have:
\begin{small} 
\begin{equation} 
    E\left[ \left\| G_{noise} \right\|^{2}_{F} \right] \leq (2BD\rho_{\sigma} + 4BC\rho_{Z})E\left[ \left\| \tilde{Z}^{(l)} - Z^{(l)} \right\|_{F}^{2} \right] + 4BCD.
    \label{eq:noisebound}
\end{equation} 
\end{small}
\end{theorem}

\begin{proof}
According to the definition of gradient noise, we can derive the Equation~\ref{eq:noisebound} as follows:
    \begin{equation}
        \begin{split}
        & E\left[ \left\| G_{noise} \right\|^{2}_{F} \right] =  E\left[ \left\| \nabla_{W^{(l)}}\tilde{L} - \nabla_{W^{(l)}}L \right\|^{2}_{F} \right] \\
        & = E\left[ \left\| (\tilde{H}^{(l)})^{T}\tilde{P}^{T}\tilde{\delta}^{(l+1)} - (H^{(l)})^{T}P^{T}\delta^{(l+1)}\right\|^{2}_{F} \right] \\
        & = E\left[ \left\| (\tilde{H}^{(l)}-H^{l})^{T}\tilde{P}^{T}\tilde{\delta}^{(l+1)} - (H^{(l)})^{T}(P^{T}\delta^{(l+1)}-\tilde{P}^{T}\tilde{\delta}^{(l+1)})\right\|^{2}_{F} \right] \\
        & \leq 2E\left[ \left\| (\tilde{H}^{(l)}-H^{l})^{T}\tilde{P}^{T}\tilde{\delta}^{(l+1)} \right\|_{F}^{2} \right] + 2E\left[ \left\| (H^{(l)})^{T}(P^{T}\delta^{(l+1)}-\tilde{P}^{T}\tilde{\delta}^{(l+1)})\right\|^{2}_{F} \right] \\
        & \leq 2E\left[ \left\| (\tilde{H}^{(l)}-H^{l})^{T}\tilde{P}^{T}\tilde{\delta}^{(l+1)} \right\|_{F}^{2} \right] + \\
        & 4E\left[ \left\| (H^{(l)})^{T}P^{T}(\delta^{(l+1)}-\tilde{\delta}^{(l+1)}) \right\|_{F}^{2} \right] + 4E\left[ \left\| (H^{(l)})^{T}(\tilde{P}^{T}-P^{T})(\tilde{\delta}^{(l+1)}) \right\|_{F}^{2} \right] \\
        & \leq 2BD\rho_{\sigma}E\left[ \left\| \tilde{Z}^{(l)} - Z^{(l)} \right\|_{F}^{2} \right] + 4BC\rho_{Z}E\left[ \left\| \tilde{Z}^{(l)} - Z^{(l)} \right\|_{F}^{2} \right] + 4BCD \\
        & = (2BD\rho_{\sigma} + 4BC\rho_{Z}E\left[ \left\| \tilde{Z}^{(l)} - Z^{(l)} \right\|_{F}^{2} \right]) + 4BCD
        \end{split}
        \label{first th}
    \end{equation}
\end{proof}

According to E.q.~\ref{eq:noisebound}, the upper bound of the expected gradient noise is decided by $\left\| \tilde{Z}^{(l)} - Z^{(l)} \right\|_{F}^{2}$, i.e., the expected value of the difference of the hidden layer embedding. We further analyze the upper bound of of this difference. 

\begin{theorem} \label{th2}
    \textbf{Upper bound of the expected hidden embeddings' difference.} Given the entire edge set $E$ and the selected edge subset $E_{s}$, we derive the following inequation:
    
    \begin{equation}
        E_{E_{s}}\left[ \left\| \tilde{Z}_{V,*} - Z_{V,*} \right\|^{2}_{F} \right] \leq \frac{1}{|E_{s}|} \sum_{(v,u)\in E} \frac{1}{p_{(v,u)}} \left\| P_{*,u}\right\|^{2}_{2}\xi^{2}.
        \label{final_min}
    \end{equation}
\end{theorem}

\begin{proof}
First, we reformulate $\tilde{Z}$ as following:
    \begin{equation}
        \tilde{Z}_{V,*} = P_{E_{s}}H_{V,*}W = P_{E}RH_{V,*}W, \label{GNN_R}
    \end{equation}
where $(P_{E}R)_{ij} = \frac{1}{|E_{s}|p_{(i,j)}}({P_{E}})_{(i,j)}$ if $(i,j) \in E_{s}$, otherwise $(P_{E}R)_{ij} = 0$; and $p_{(i,j)}$ is the probability of edge $(i,j)$ to be selected. 

Then $E\left[ P_{E}RHW \right] = P_{E}HW$, and we can derive Equation~\ref{final_min} as follows:
        \begin{equation}
            \begin{split}
            & E_{E_{s}}\left[ \left\| \tilde{Z}_{V,*} - Z_{V,*} \right\|^{2}_{F} \right]
            = E_{E_{s}}\left[ \left\| P_{E}RH_{V,*}W - P_{E}H_{V,*}W \right\|^{2}_{F} \right] \\
            & = \sum_{v \in V} E_{E_{s}}\left[ \left\| \frac{1}{|E_{s}|}\sum_{(v,u)\in E_{s}}\frac{1}{p_{(v,u)}}P_{v,u}H_{u,*}W - \sum_{(v,w)\in E} P_{v,w}H_{w,*}W\right\|^{2}_{2} \right] \\
            & = \frac{1}{|E_{s}|} \sum_{v \in V} E_{E_{s}}\left[ \left\| \frac{1}{p_{(v,u)}}P_{v,u}H_{u,*}W - \sum_{(v,w)\in E} P_{v,w}H_{w,*}W\right\|^{2}_{2} \right] \\ 
            & = \frac{1}{|E_{s}|} \sum_{v \in V} \sum_{(v,u)\in E}p_{(v,u)}  \left\| \frac{1}{p_{(v,u)}}P_{v,u}H_{u,*}W - \sum_{(v,w)\in E} P_{v,w}H_{w,*}W\right\|^{2}_{2}  \\ 
            & = \frac{1}{|E_{s}|} \sum_{v \in V} \left[ \sum_{(v,u)\in E} \frac{1}{p_{(v,u)}} \left\| P_{v,u}H_{u,*}W\right\|^{2}_{2} - \left\| \sum_{(v,w)\in E} P_{v,w}H_{w,*}W\right\|^{2}_{2}\right]\\
            & \leq \frac{1}{|E_{s}|} \sum_{v \in V} \sum_{(v,u)\in E} \frac{1}{p_{(v,u)}} \left\| P_{v,u}H_{u,*}W\right\|^{2}_{2} \\
            & = \frac{1}{|E_{s}|} \sum_{v \in V} \sum_{(v,u)\in E} \frac{1}{p_{(v,u)}} \left\| P_{v,u}\right\|^{2}_{2} \xi^{2} \\
            & = \frac{1}{|E_{s}|} \sum_{(v,u)\in E} \frac{1}{p_{(v,u)}} \left\| P_{*,u}\right\|^{2}_{2} \xi^{2}
            \end{split}
            \label{second th}
        \end{equation}
\end{proof}

As illustrated in E.q.~\ref{final_min}, we find the upper bound of hidden embeddings' difference is related to edge sampling probability $p_{v,u}$. By combining Equation\ref{eq:noisebound} and removing the constants that are hard to calculate, we can formulate a gradient noise optimization problem and minimize the value of the noise by using the following constraint:

\begin{equation}
    s.t. \sum_{(v,u)\in E} p_{(v,u)} = 1 \label{conditoin}
\end{equation}

Based on E.q.~\ref{final_min} and E.q.~\ref{conditoin}, by using the Lagrange function, we derive the edge sampling probability as defined in E.q.~\ref{final_prob}, and the probability can reduce the gradient noise in the spanning subgraph training.

\subsection{Two-step Edge Sampling Method}\label{sec:two-step-sampling}
The probabilities defined by E.q.~\ref{prob:var} and~\ref{final_prob} are non-uniform. It is a challenge to fast sample non-uniform distribution in large sample space~\cite{non_uniform_sample}. An efficient sampling method, Alias sampling~\cite{walker1977efficient}, requires massive memory and entails the high cost of building data structures. In this paper, we propose a simple but effective approximate sampling method -- two-step edge sampling to speed up the quality-aware edge selection process. 

In the first step, \framework reduces the sample space by randomly sampling an edge set, denoted as $e_{t}$, in iteration $T_{t}$. This step confines the final selected edges focusing on the edge set $e_{t}$ rather than the entire edge set $E$. In the second step, \framework samples $e_{t}^{'}$ from the edge set $e_{t}$ according to the probability defined by E.q.~\ref{prob:var} and~\ref{final_prob}. 

Algorithm~\ref{alg:sampling} illustrates the procedures of quality-aware edge selection with two-step sampling. $RandomSample$ is the function that randomly selects a fixed number of edges from the input graph (Line 1). $vm$ and $gnr$ respectively represents the variance-minimized sampling strategy and the gradient noise-reduced sampling strategy (Line 2,5). $ProbabilitySample$ is the function that selects a fixed number of edges from the input graph with certain edge sampling probability (Lines 3,6).

\begin{algorithm}[ht]
\caption{ Quality-aware Edge Selection with Two-step Sampling.}
\label{alg:sampling}
\begin{algorithmic}[1] 
    \REQUIRE The graph $G$, first-step sampling size $S_{1}$, second-step sampling size $S_{2}$, the type of the edge sampling probability $prob$;
    \ENSURE Edges $\Delta G$;
    \STATE $G^{'} = RandomSample(G, S_{1})$; $//$ $|G^{'}| = S_{1}$
    \IF{$prob = `vm'$}
        \STATE $G^{''} = ProbabilitySample(G^{'}, S_{2}, `vm')$; $//$ $|G^{''}| = S_{2}$
    \ENDIF
    \IF{$prob = `gnr'$}
        \STATE $G^{''} = ProbabilitySample(G^{'}, S_{2}, `gnr')$; $//$ $|G^{''}| = S_{2}$
    \ENDIF
    \STATE $\Delta G = G^{''}$;
    \RETURN $\Delta G$;
\end{algorithmic}
\end{algorithm}

Here, we discuss the advantages of two-step sampling with a memory-efficient non-uniform sampling method, which first constructs a cumulative probability array, then uses random numbers to select elements. The time complexity entailed by the first step of random sampling from the entire edge set is $O(|e|)$. The time complexity of the second step of weighted sampling is about $O(|e|+|e^{'}|log(|e|))$. Therefore, the total time complexity of the two-step sampling is $O(|e|)+O(|e^{'}|log(|e|))$. However, if we directly sample $|e^{'}|$ edges from entire edge set, the time complexity is $O(|E|+|e^{'}|log(|E|))$. In practice, $|e|$ is typically several to ten times $|e^{'}|$, while $|E|$ can be up to a hundred times larger than $|e|$. Therefore, the time cost of the two-step sampling is lower than that of direct sampling.


\section{Connection to Curriculum Learning} \label{cl}
In this section, we analyze the connection between \framework and curriculum learning.
To our knowledge, curriculum learning increases the robustness of the learned model against noisy training samples by training samples from easy to hard. An intuitive explanation is that curriculum learning spends less time with the harder (noisy) samples to achieve better robustness. 

\framework incorporates the principles of curriculum learning by constructing different graph structures (i.e., spanning subgraphs) during the learning process. \framework not only mirrors the educational strategy of progressing from easy to hard lessons, but also aligns with the model's need to first grasp fundamental concepts before tackling more challenging tasks. The detailed discussion is as follows.

First, in \framework, the empty graph at the beginning can be regarded as the simplest `course'. In the training process, edges are gradually added to the graph. This progressive learning process helps the model master basic structural information first, and then learn more complex graph structures, which helps the model to learn more robust and avoid overfitting. 

Second, through the Quality-aware Edge Selection, we prioritize edges that are more significant for model training, to help minimize feature variance and reduce gradient noise. Edges with smaller feature variance mean that the aggregated features are more consistent. Also, edges with less gradient noise mean that they can help the model learn more stable. This is similar to the `from easy to hard' in curriculum learning, where we initially learn the data that will be more beneficial for subsequent learning. 

\section{Experimental studies}
In this section, we start with the descriptions of experimental settings, which cover the datasets and configurations used in the experiments. Then, we evaluate the performance of \framework by comparing it with full-graph training methods, conduct ablation studies to verify the effectiveness of the proposed techniques, and study the efficiency of \framework. Finally, we also compare \framework with mini-batch training methods to demonstrate that generally \framework is able to achieve high accuracy. In addition, due to the limited space, we put the results of parameter sensitivity in Section 4 of the Supplementary Material. 

\subsection{Experimental Setups} \label{Experimental setup}
\begin{table}[htbp]
\caption{Dataset statistics}
\begin{center}
\begin{tabular}{|c|c|c|c|c|}
\hline
\textbf{Dataset}&\multicolumn{4}{|c|}{\textbf{Dataset attributions}} \\
\cline{2-5} 
\textbf{name} & \textbf{\textit{\#Nodes}} & \textbf{\textit{\#Edges}} & \textbf{\textit{Features}} & \textbf{\textit{Classes}} \\
\hline
Ogbn-proteins & 132,534 & 79,122,504 & 8 & 112 \\
Reddit & 232,965 & 114,615,892 & 602 & 41 \\
Amazon & 1,598,960 & 264,339,468 & 200 & 107  \\
Ogbn-products & 2,449,029 & 126,167,053 & 100 & 47 \\
\hline
\end{tabular}
\label{tab1}
\end{center}
\end{table}

\setcounter{footnote}{0}

\textbf{Environments \& Datasets.} We implemented \framework with PyTorch 2.0.1, and the code is released\footnote{https://github.com/guxizhi/SpanGNN}. 
We evaluate the performance of \framework using two common GNN models including GCN~\cite{kipf2016semi} and SAGE~\cite{hamilton2017inductive} with the mean aggregator.
All experiments are conducted on NVIDIA RTX A6000. We use four large graph datasets. Table \ref{tab1} lists the summary of the datasets. 

\textbf{Performance Metrics \& Evaluation Protocol}. Accuracy is used to measure the effectiveness of \framework on Reddit and Ogbn-products datasets, F1-score is used on Amazon, and AUC-ROC is used on Ogbn-proteins. 
All performance metrics are calculated on the validation set and the results are the average of three times experiments. Furthermore, we conduct experiments under different edge ratios to verify the memory-efficiency of \framework. 

\textbf{Baselines.} 
1) \textbf{Full-graph}. It is a naive full-graph training method, but consumes heavy GPU memory. 
2) \textbf{DropEdge}. It has good scalability for training on large graphs.
3) \textbf{GraphSAGE}, \textbf{ClusterGCN} and \textbf{GraphSAINT} are selected as the representations of the mini-batch training.

Additionally, \framework equipped with variance-minimized sampling and gradient noise-reduced sampling respectively are denoted by \textbf{\framework-F} and \textbf{\framework-G}. 

\begin{figure}[!t]
    \centering
    \includegraphics[width=\textwidth]{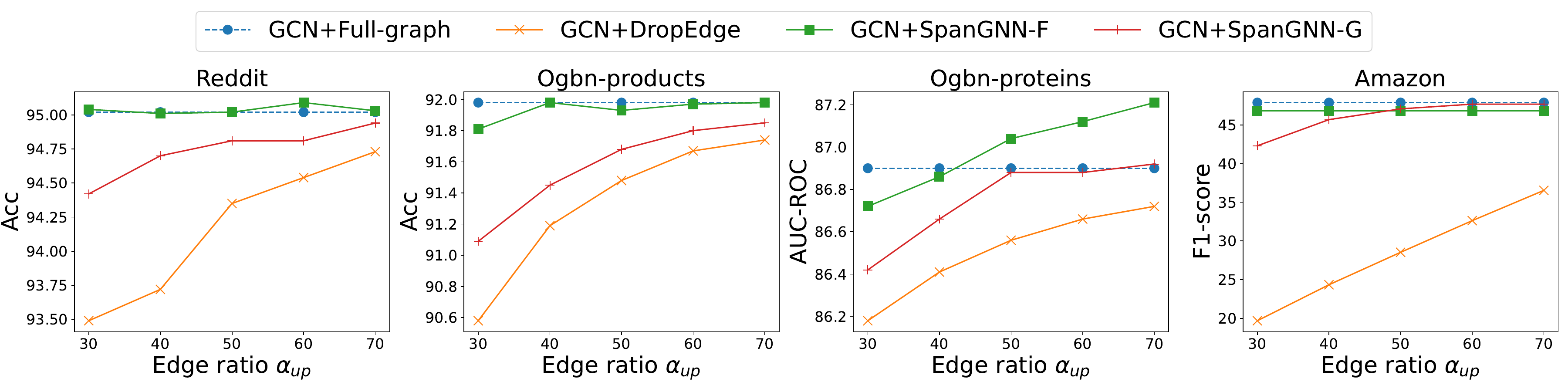}
    \includegraphics[width=\textwidth]{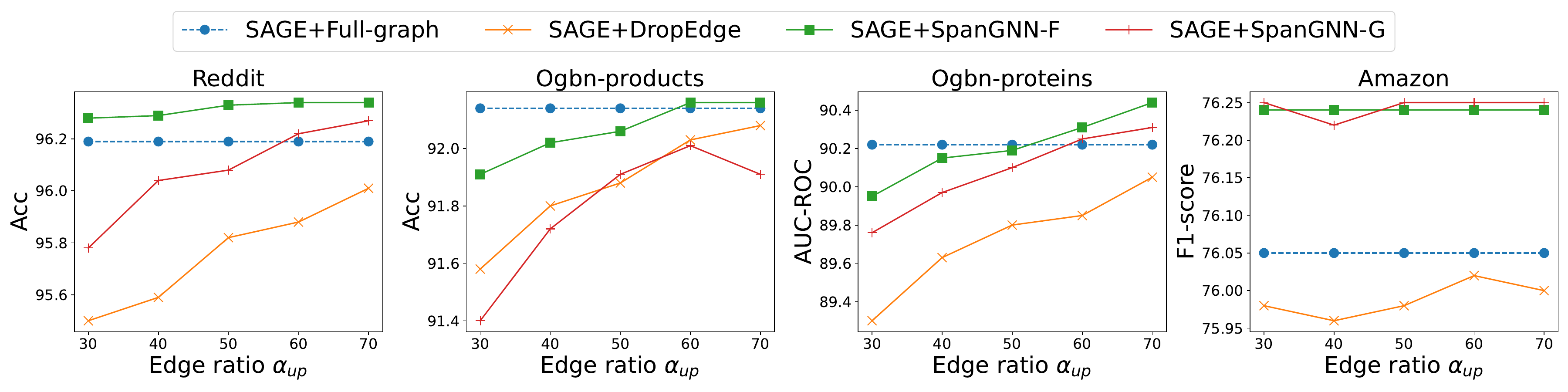}
    \caption{The performance of the training methods on GCN (Up-side) and SAGE (Down-side) with various edge ratios.}
    \label{gcn_sage_fd}
    \vspace{-1em}
\end{figure}

\subsection{Performance of \framework } \label{efficiency exp}
\subsubsection{Comparison of model accuracy.}
Figure~\ref{gcn_sage_fd} illustrates the model accuracy of \framework, Full-graph and DropEdge on GCN and SAGE with various edge ratios $\alpha_{up}$, which are set from 0.3 to 0.7. We see that \framework-F's performance is similar to or even better than Full-graph. For example, on Reddit, the accuracy of \framework-F is higher than the one of Full-graph with SAGE, regardless of $\alpha_{up}$. On Ogbn-proteins, as $\alpha_{up}$ gets larger, the AUC-ROC of \framework-F gradually exceeds the one of Full-graph. The exception is on Amazon, where \framework-G is better than \framework-F as $\alpha_{up}$ gets larger. This is because \framework-F's sampling probability on Amazon is extremely skewed, and it is caused by the fact that few edges are connected by two low-degree nodes. These minority edges are given larger weight during selection, causing them to be selected repeatedly in every edge selection. It is hard to obtain enough edges for the spanning subgraph (i.e., the edge ratio of a spanning subgraph is hard to reach $\alpha_{up}$) and results in a decrease of F1-score. This problem also reduces the size of peak memory usage. In Figure~\ref{gcn_sage_memory}, we see that the peak memory usage of \framework-F is stable with respect to different edge ratios on Amazon.

Compared to \framework, DropEdge suffers from the decrease in model performance more seriously. On Reddit, DropEdge losses the accuracy by up to 1.5\% on GCN and by up to 0.8\% on SAGE. Even worse, DropEdge severely damages the model's F1-score on Amazon by more than 25\%. Overall, \framework is better at ensuring model's performance compared to DropEdge.

\begin{figure}[!t]
    \centering
    \includegraphics[width=\textwidth]{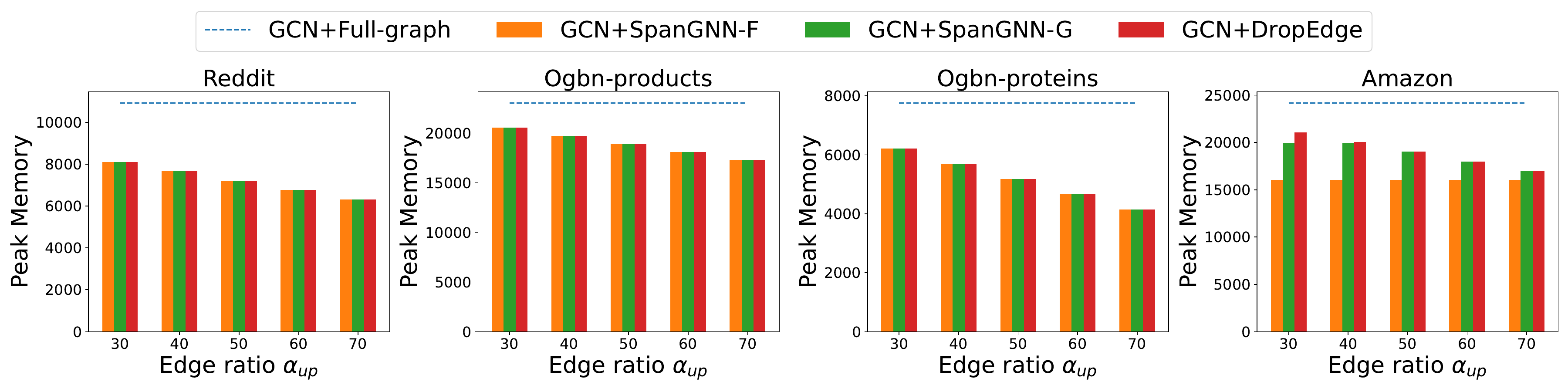}
    \includegraphics[width=\textwidth]{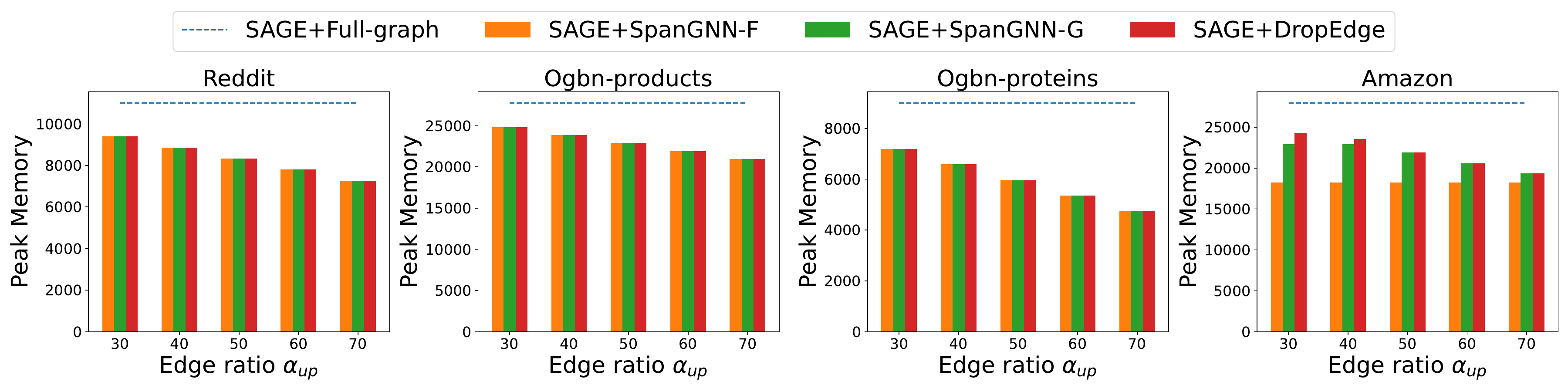}
    \caption{Peak Memory Usage on GCN(Up-side) and SAGE(Down-side).}
    \label{gcn_sage_memory}
    \vspace{-1em}
\end{figure}

\subsubsection{Comparison of peak memory usage.}
Here we compare the peak memory usage among \framework, Full-graph, DropeEdge. As shown in Figure~\ref{gcn_sage_memory}, we see that reducing the number of edges effectively reduces the size of peak memory by comparing \framework and Full-graph. There is no significant difference between \framework and DropEdge, since they drop the same size of edges. In addition, the percentage of peak memory saved is also independent of the model. By using only 30\% edges, \framework and DropEdge can reduce the peak memory usage by 42\%, 25\% and 47\% on Reddit, Ogbn-products and Ogbn-proteins, respectively. Note that on Amazon, due to the actual edge ratio cannot achieve $\alpha_{up}$, which is discussed in the ``Comparison of model accuracy'', \framework-F has less peak memory overhead. 

\begin{figure}[!t]
    \centering
    \includegraphics[width=\textwidth]{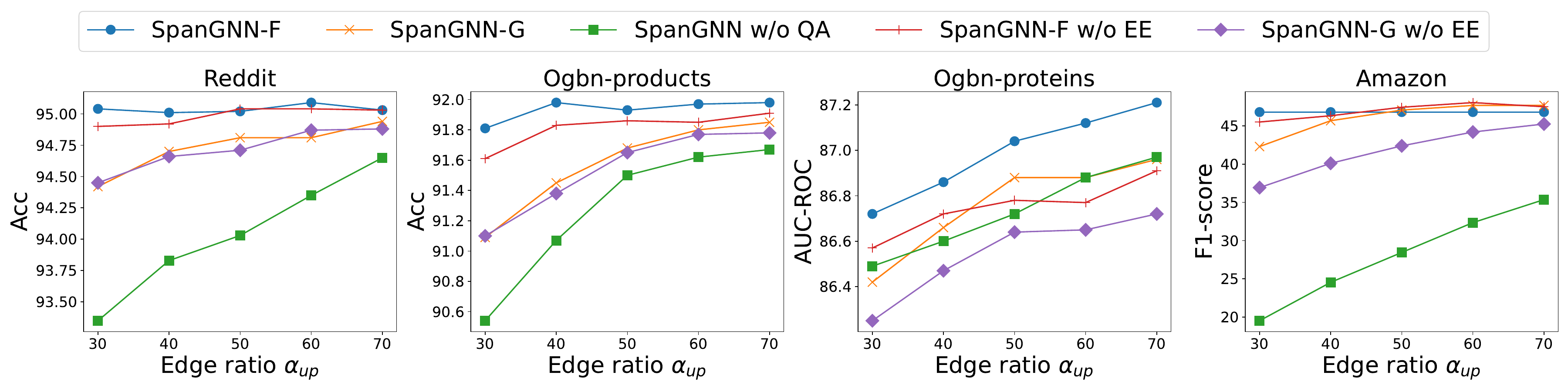}
    \includegraphics[width=\textwidth]{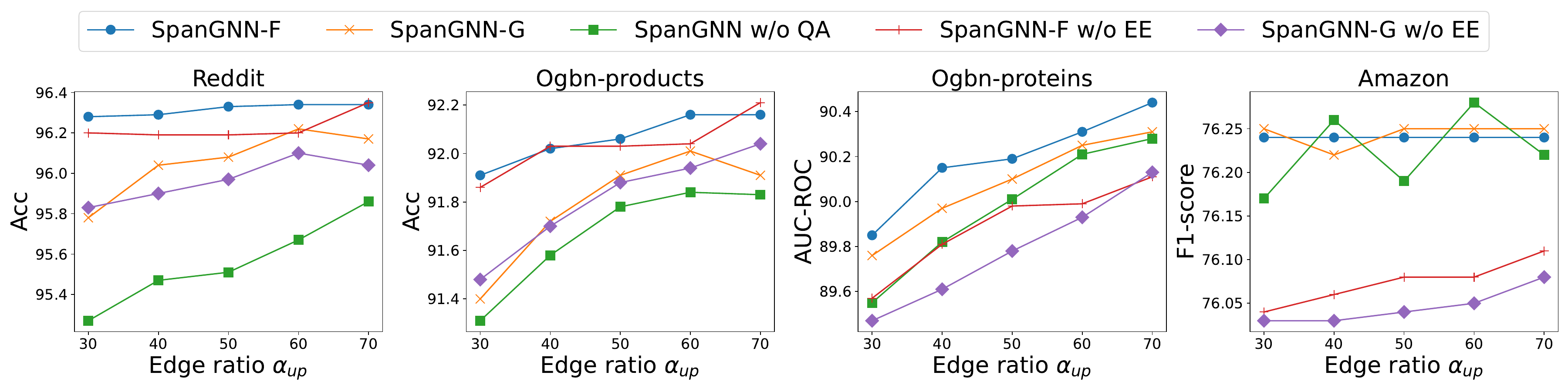}
    \caption{Ablation studies on GCN(Up-side) and SAGE(Down-side)}
    \label{gcn_sage_ab}
    \vspace{-1em}
\end{figure}
\subsection{Ablation Studies} 
\subsubsection{Effectiveness of the framework.}
In order to verify the effectiveness of the principles of curriculum learning used by \framework, we compare the model accuracy between \framework and \framework w/o EE. Instead of the empty graph, \framework w/o EE is initialized by the graph with $\alpha_{up}|G|$ edges, which are selected by quality-aware edge selection. The results are shown in Figure~\ref{gcn_sage_ab}.

The results indicate that \framework improves the model performance on various datasets by adopting the curriculum learning principles. On Ogbn-proteins, it is clear that \framework outperforms \framework w/o EE and the improvement is around 0.2\%. On other datasets, depending on the based model and the value of $\alpha_{up}$, \framework is generally better than or equal to \framework w/o EE.

\subsubsection{Effectiveness of quality-aware edge selection.}
In order to verify the effectiveness of variance-minimized sampling and gradient noise-reduced sampling strategies, we compare the model performance among \framework-G, \framework-F, and \framework w/o QA. Here \framework w/o QA applies random sampling instead of quality-aware sampling. The results are shown in Figure~\ref{gcn_sage_ab}.


Generally, \framework-G and \framework-F have better performance than \framework w/o QA. The advantage can reach 1.5\% on Reddit and even more than 20\% on Amazon. In certain cases, \framework w/o QA might outperform \framework. 
As discussed in the ``Comparison of model accuracy'', on Amazon, the spanning subgraph in \framework is easy to contain fewer edges than the required ones defined by $\alpha_{up}$ because of the skewed sampling probability. However, \framework w/o QA can successfully reach $\alpha_{up}$, and contain sufficient edges. Therefore, \framework does not always perform better than \framework w/o QA on SAGE. Overall, we conclude that the variance-minimized sampling and the gradient noise-reduced sampling generally plays important roles in improving performance of \framework. 

\begin{figure*}[!t]
    \centering
    \subfigure[GCN]{
    \begin{minipage}[t]{0.5\linewidth}
    \centering
    \includegraphics[width=\textwidth]{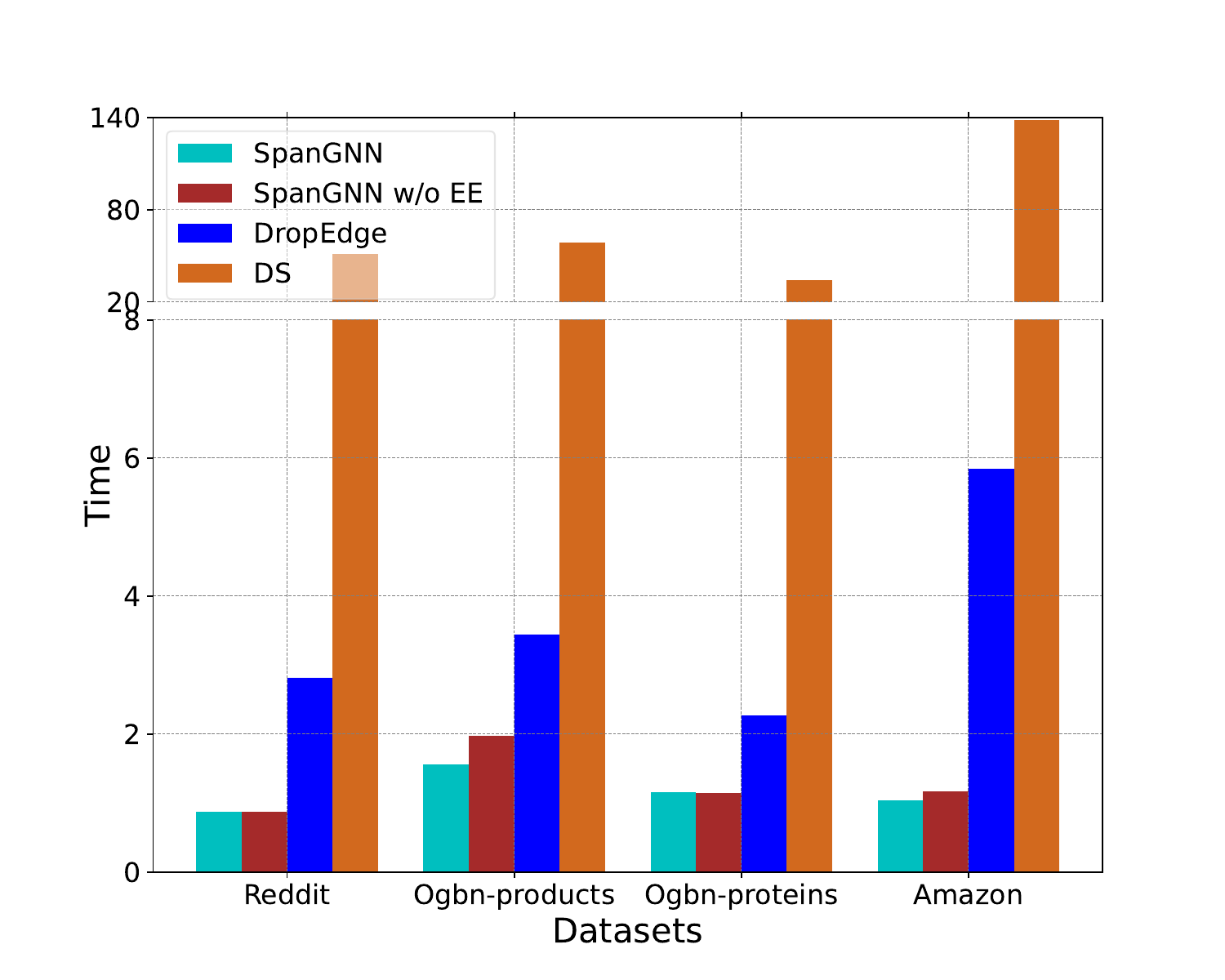}
    \label{time_consume_gcn}
    \end{minipage}%
    }%
    \subfigure[SAGE]{
    \begin{minipage}[t]{0.5\linewidth}
    \centering
    \includegraphics[width=\textwidth]{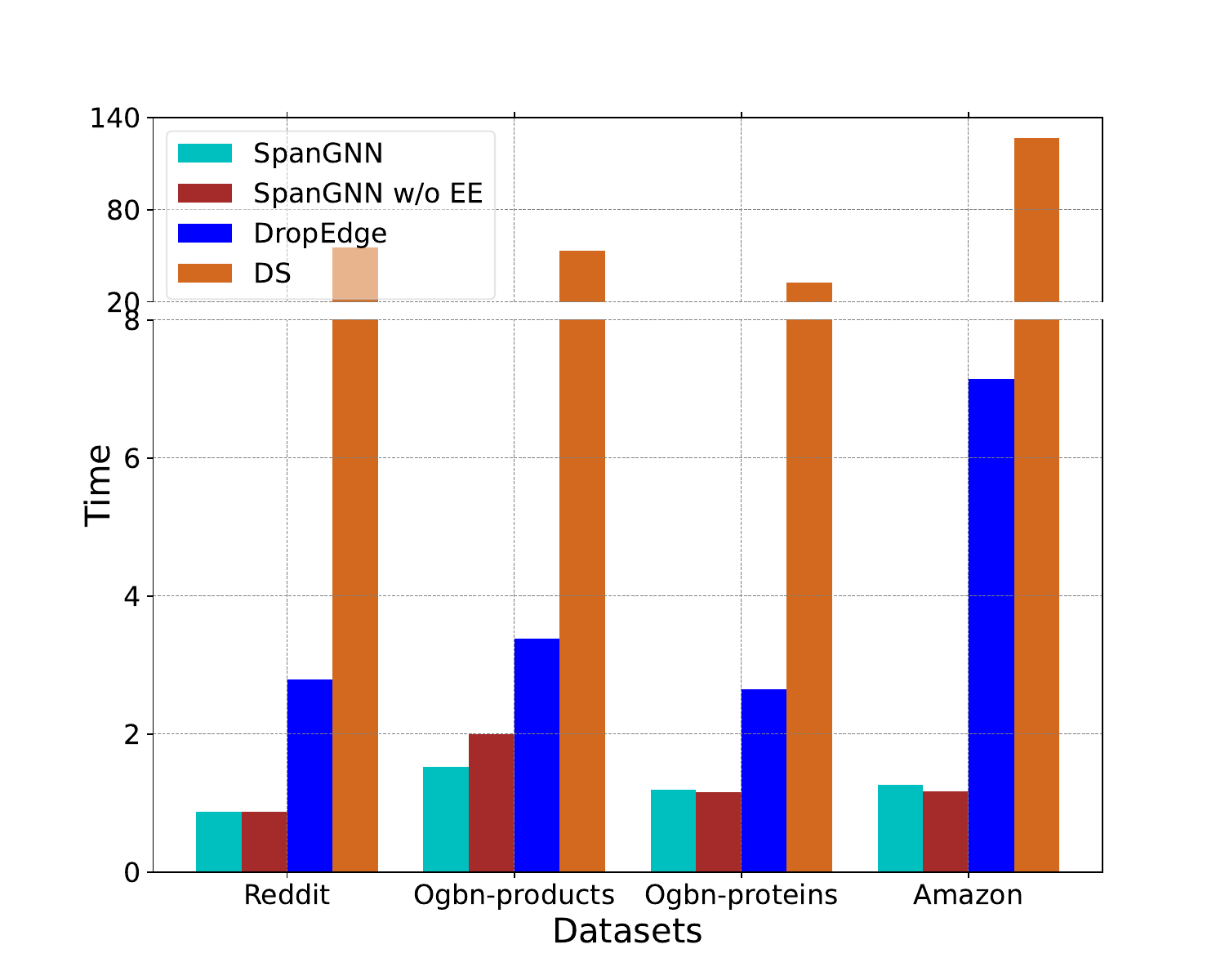}
    \label{time_consume_sage}
    \end{minipage}%
    }%
    \centering
    \caption{The average time cost of generating spanning subgraphs.}
    \label{time_consume}
    \vspace{-1em}
\end{figure*}
\subsection{Efficiency of \framework} \label{efficiency exp}

In this section, we demonstrate the efficiency of \framework by comparing the average time cost of generating spanning subgraphs. Figure~\ref{time_consume} illustrates the results of \framework, \framework w/o EE, DropEdge and direct sampling from the entire graph based on quality-aware edge selection (DS). To guarantee the fairness of the comparison, all methods have the same edge ratio (i.e., $\alpha_{up}=0.3$). 

As we can see, \framework is more efficient than other frameworks, and integrating curriculum learning principle does not destroy the execution efficiency. Specifically, compared to DropEdge, \framework speeds up from 1.95x to 5.65x on different datasets. 
As analyzed in Section~\ref{sec:two-step-sampling}, the time complexity of generating spanning subgraphs in \framework is proportional to the number of first-step sampling edges $e$. Due to dropping a lot of edges each iteration (i.e., $\alpha_{up}=0.3$), DropEdge entails much more time cost than \framework.
When compared to DS, \framework speeds up from 26.99x to 132.08x.  This is because the time complexity of DS is proportional to the number of edges in the entire graphs, which can reach hundreds of times that of $e$. 

\begin{table}[ht]
\caption{The comparison of model performance with mini-batch training methods. Note that \framework's results are determined by taking the best one among different edge ratios.}
\begin{center}
\scalebox{0.9}{
\begin{tabular}{|c|c|c|c|c|c|}
\hline
\cline{3-6}
\textbf{GNN} & \textbf{Model} & \textbf{Reddit} & \textbf{Ogbn-products} & \textbf{Amazon} & \textbf{Ogbn-proteins} \\
\cline{3-6}
& & Acc & Acc & F1-score & AUC-ROC \\
\hline
& \framework-G & 95.26 & \underline{91.50} & \textbf{47.79} & \underline{87.11}\\
& \framework-F & \underline{95.46} & \textbf{91.68} & 46.78 & \textbf{87.19}\\
GCN & GraphSAGE & 91.99 & 90.18 & 28.73 & 71.31\\
& ClusterGCN & 92.05 & 89.94 & \underline{46.86} & 79.30\\
& GraphSAINT & \textbf{96.53} & 90.14 & 7.50 & 80.12\\
\hline
& \framework-G & 96.51 & \underline{91.33} & \underline{76.29} & \underline{90.36}\\
& \framework-F & \underline{96.62} & \textbf{91.90} & 76.26 & \textbf{90.49}\\
SAGE & GraphSAGE & 94.55 & 90.57 & 72.99 & 82.35\\
& ClusterGCN & 94.74 & 90.55 & \textbf{77.43} & 83.54\\
& GraphSAINT & \textbf{97.46} & 90.15 & 75.21 & 85.35\\
\hline
\end{tabular}}
\label{to mini}
\end{center}
\vspace{-1em}
\end{table}

\subsection{Performance of \framework Compared to Mini-batch Training}

In this section, we compare \framework with different mini-batch training methods in terms of model performance. The memory usages of mini-batch training is related to the batch size, and it is more flexible than \framework in terms of memory consumption. However, as shown in Table~\ref{to mini}, generally \framework achieves better performance than the mini-batch methods. On Ogbn-products and Ogbn-proteins, \framework always outperforms the mini-batch methods and its improvements can achieve 1.7\% and 7.0\% respectively. On other datasets, \framework is either the best one or the second best but very close to the best one. Therefore, compared to the mini-batch training, \framework achieves high model performance.

\subsection{Parameter Sensitive Analysis of Two-step Sampling}

\begin{figure*}[ht]
    \centering
    \includegraphics[width=\textwidth]{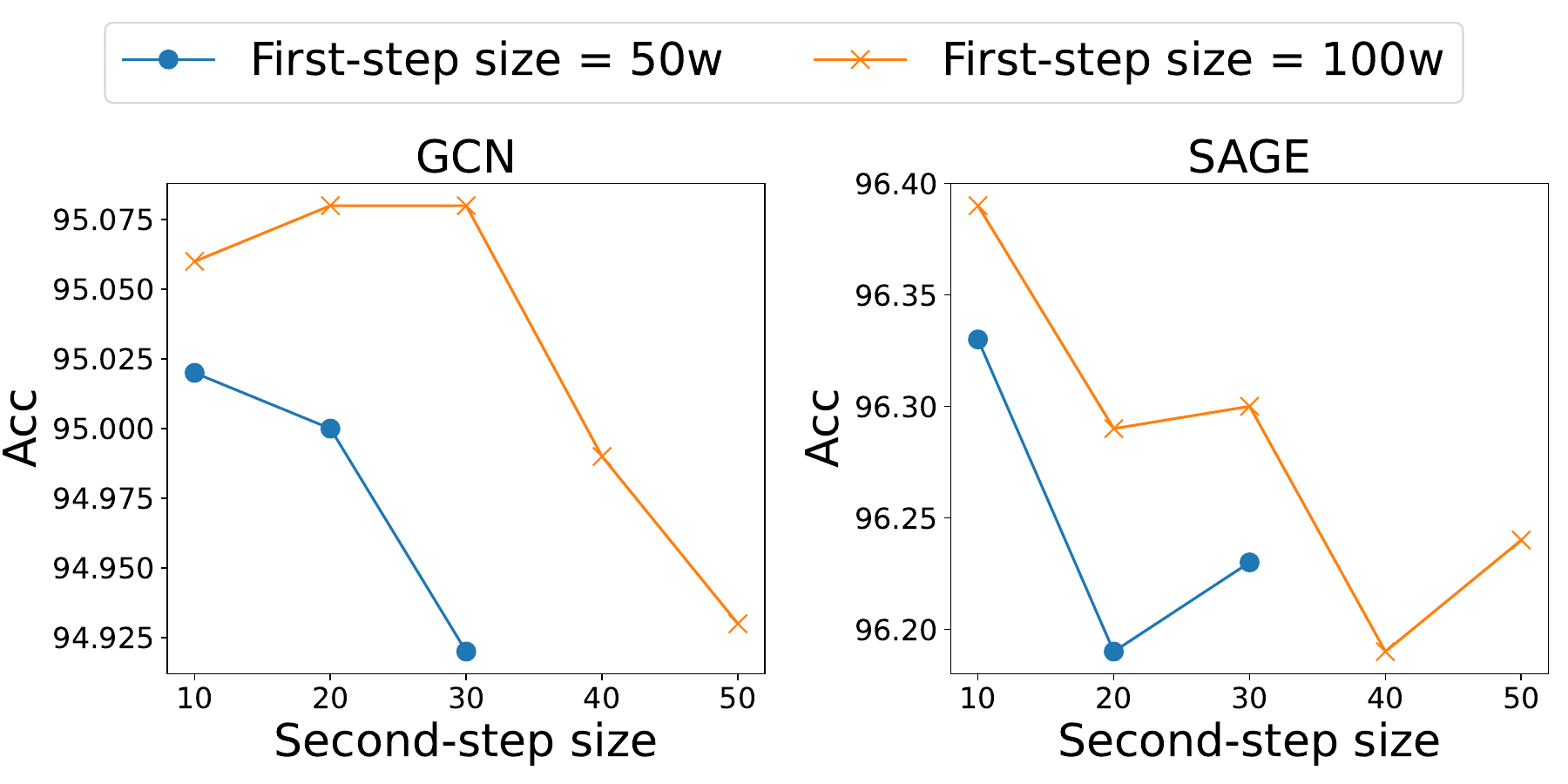}
    \centering
    \caption{the influence of two-step sampling sizes on GCN(left) and SAGE(right)}
    \label{sensitive}
\end{figure*}

In order to fully understand the influence of curriculum learning principles in \framework, we conduct parameter sensitivity analysis of two-step sampling. We test \framework-F on the Reddit dataset, and set the first step sampling size $|e|$ as 50w and 100w, and the second step sampling size $|e^{'}|$ varies from 10w to 50w. 

Figure~\ref{sensitive} presents the results. We see that GNNs' performance decreases obviously as second-step size increases. When the base model is GCN, the accuracy decrease by up to 0.15\%. Also, the similar phenomenon exists when the base model is SAGE. The accuracy shows a reduction of as much as 0.20\%. The reason of this trend is discussed as follows: the speed of edge growth in the spanning subgraph is directly determined by the size of $|e^{'}|$, which actually reflects the size of the curriculum learning's step. The larger the $|e^{'}|$ is, the larger the curriculum learning's step is. As a results, a larger $|e^{'}|$ coarsens the granularity of curriculum learning and reduces the positive effect from curriculum learning.

What's more, with a fixed value of $|e^{'}|$, the higher the value of $|e|$, the higher the accuracy is. This is because that when the ratio of $|e^{'}|$ to $|e|$ increases (i.e., small $|e|$), more low-probability edges are likely to be selected. And it implies that difficult samples are involved to the training early, thus destroying the positive effect of curriculum learning. 

\section{Related work} \label{Releted work}

\subsection{Memory-Efficient Graph Neural Networks}

Mini-batch training is an effective approach to reduce the memory consumption. Existing works do a lot of exploration on sampling methods with mini-batch training approach. The works~\cite{hamilton2017inductive},~\cite{chen2018stochastic} apply node-level sampling to select a set of nodes in neighbors. In this way, it reduce the number of each node's neighbors in the phase of aggregation. However, it can not resolve the problem of `neighbor explosion' when GNNs goes deeper. The works~\cite{2018FastGCN},~\cite{2019Ladies} apply layer-level sampling to select a fixed number of nodes in each GNN layer. Since this type of sampling methods use fixed number of nodes in the each layer, it can alleviate the `neighbor explosion'. However, FastGCN~\cite{2018FastGCN} suffers from unbalanced receptive fields. LADIES~\cite{2019Ladies} tracks each node's neighbors in the previous layer and calculates an importance estimator, but causes much overhead. The works~\cite{2019Cluster},~\cite{graphsaint} apply subgraph-level sampling to limit the aggregation field to a subgraph. ClusterGCN~\cite{2019Cluster} partitions the graph into a set of clusters and then randomly combines partitions to be a mini-batch. GraphSAINT~\cite{graphsaint} directly forms subgraphs with overlapping nodes among mini-batches. However, compared to Full-graph training, mini-batch training incurs information loss.

Even though almost no work explicitly discusses using spanning subgraph to reduce peak memory usage during GNNs training, there exist some close works. DropEdge~\cite{Dropedge} randomly removes a certain percentage of edges from the original input graph in each epoch. It alleviates the problem of over-smoothness~\cite{graziani2023no,bause2024sides} and over-fitting and can be considered as a strategy that uses spanning subgraph. TADropEdge~\cite{TADropedge} additionally considers the factors of graph structure. They analyze the graph connectivity and gives larger weight to keep inter-cluster edges in GNNs training.
NeuralSparse~\cite{NeuralSparse} applies a deep neural network to learn how to sparse graphs with the feedback of downstream prediction tasks. It improves generalization ability by removing potentially task-irrelevant edges. SGCN~\cite{li2020sgcn} also considers sparsification as an optimization problem and applies ADMM-based solution to solve it. But these works focus on improving the prediction results, overlooking the problem of peak memory usage. In addition, some other works directly delete or change the structure of GNNs model. SGC~\cite{wu2019simplifying} reduces redundant calculations by deleting the non-linear activation function between GCN layers. PPRGo~\cite{bojchevski2020scaling}, by calculating the influence matrix, avoids the overhead of collecting multi-hop neighbors. However, they fundamentally change the characteristics of GNNs, and cannot directly be applied to existing models.

\subsection{Curriculum Learning on GNN}
Recently, curriculum learning is introduced into GNNs training and achieves performance improvement in GNN models. CurGraph~\cite{wang2021curgraph} introduces curriculum learning to train GNNs with graphs in ascending order of difficulty. This method uses the informax technique for graph-level embeddings and a neural density estimator to model the embedding distributions. After calculating the difficulty scores of graphs, it first exposes GNN models to easy graphs and moves on to harder ones. They focus on the prediction of graphs. CLnode~\cite{wei2023clnode} defines the difficulty of samples at the level of the node and applies various pacing functions to train GNNs from easy-to-hard. It measures nodes' difficulty from the perspective of neighborhoods and features. RCL~\cite{zhang2024curriculum} considers that connections of nodes significantly affect the curriculum learning. It distinguishes the level of difficulty for edges and gradually incorporates more information at the level of edges. However, neither CLnode nor RCL takes the memory usage into account, and they need to train GNNs on the entire graphs. Differently, our work sets an upper bound of the number of edges and designs difficulty scoring function by fully considering the impact of the spanning subgraph.

\section{Conclusion}
In this paper, we proposed \framework that carries out GNNs training on large-scale graphs efficiently by using spanning subgraphs and integrating the principles of curriculum learning. \framework consists of two main components to limit the memory overhead and ensure the model performance. Quality-aware edge selection samples beneficial edges for spanning subgraph GNN training and follows the manner of curriculum learning to add edges for training. Graph update determines the size of the spanning subgraph at each epoch to control the peak memory. Overall, we provide an efficient large-scale GNN training method that can reduce memory overhead and maintain the model performance.


%
%
%
\bibliographystyle{splncs04}
\bibliography{references}

\end{document}